\documentclass{article}
\usepackage{mathptmx}
\usepackage[T1]{fontenc}
\usepackage[utf8]{inputenc}
\usepackage{amsmath}
\usepackage{amsthm}
\usepackage{amssymb}
\usepackage{graphicx}
\usepackage{microtype}

\makeatletter
\theoremstyle{plain}
\newtheorem{thm}{\protect\theoremname}
\theoremstyle{plain}
\newtheorem*{thm*}{\protect\theoremname}

\usepackage{amsthm}

\usepackage{natbib}
 \usepackage[preprint]{nips_2018}

\usepackage{url}
\usepackage{booktabs}
\usepackage{amsfonts}
\usepackage{nicefrac}

\usepackage{caption}
\usepackage{color}
\usepackage{xcolor}

\title{Understanding training and generalization in deep learning by Fourier analysis}

\author{
  Zhi-Qin John Xu\thanks{This work is done while Xu is a visiting member at Courant Institute of Mathematical Sciences, New York University, New York, United States.} \\
  New York University Abu Dhabi\\
  Abu Dhabi 129188, United Arab Emirates \\
  \texttt{zhiqinxu@nyu.edu} \\
}

\makeatother

\providecommand{\theoremname}{Theorem}

\begin{document}

\maketitle 
\begin{abstract}
\textbf{Background}: It is still an open research area to theoretically
understand why Deep Neural Networks (DNNs)---equipped with many more
parameters than training data and trained by (stochastic) gradient-based
methods---often achieve remarkably low generalization error. \textbf{Contribution}:
We study DNN training by Fourier analysis. Our theoretical framework
explains: i) DNN with (stochastic) gradient-based methods often endows
low-frequency components of the target function with a higher priority
during the training; ii) Small initialization leads to good generalization
ability of DNN while preserving the DNN's ability to fit any function.
These results are further confirmed by experiments of DNNs fitting
the following datasets, that is, natural images, one-dimensional functions
and MNIST dataset.
\end{abstract}

\section{Introduction\label{sec:Introduction}}

\paragraph*{Background}

Deep learning has achieved great success as in many fields (\cite{lecun2015deep}).
Recent studies have focused on understanding why DNNs, trained by
(stochastic) gradient-based methods, can generalize well, that is,
DNNs often fit the test data well which are not used for training
in practice. Counter-intuitively, although DNNs have many more parameters
than training data, they can rarely overfit the training data in practice. 

Several studies have focused on the local property (sharpness/flatness)
of loss function at minima (\cite{hochreiter1995simplifying}) to
explore the DNN's generalization ability. \cite{keskar2016large}
empirically demonstrated that with small batch in each training step,
DNNs consistently converge to flat minima, and lead to a better generalization.
However, \cite{dinh2017sharp} argued that most notions of flatness
are problematic. To this end, \cite{dinh2017sharp} used deep networks
with rectifier linear units (ReLUs) to theoretically show that any
minimum can be arbitrarily sharp or flat without specifying parameterization.
With the constraint of small weights in parameterization, \cite{wu2017towards}
proved that for two-layer ReLU networks, low-complexity solutions
lie in the areas with small Hessian, that is, flat and large basins
of attractor (\cite{wu2017towards}). They then concluded that a random
initialization tends to produce starting parameters located in the
basin of flat minima with a high probability, using gradient-based
methods. 

Several studies rely on the concept of stochastic approximation or
uniform stability (\cite{bousquet2002stability,hardt2015train}).
To ensure the stability of a training algorithm, \cite{hardt2015train}
assumed loss function with good properties, such as Lipschitz or smooth
conditions. However, the loss function of a DNN is often very complicated
(\cite{zhang2016understanding}). 

Another approach to understanding the DNN's generalization ability
is to find general principles during the DNN's training. Empirically,
\cite{arpit2017closer} suggested that DNNs may learn simple patterns
first in real data, before memorizing. \cite{xu_training_2018} found
a similar phenomenon empirically, which is referred to \emph{Frequency
Principle (F-Principle)}, that is, for a low-frequency dominant function,
DNNs with common settings first quickly capture the dominant low-frequency
components while keeping the amplitudes of high-frequency components
small, and then relatively slowly captures those high-frequency components.
F-Principle can explain how the training can lead DNNs to a good generalization
empirically by showing that the DNN prefers to fit the target function
by a low-complexity function \cite{xu_training_2018}. \cite{rahaman2018spectral}
found a similar result that “Lower Frequencies are Learned First”.
To understand this phenomenon, \cite{rahaman2018spectral} estimated
an inequality that the amplitude of each frequency component of the
DNN output is controlled by the spectral norm of DNN weights theoretically.
\cite{rahaman2018spectral} then show that the spectral norm\footnote{In \cite{rahaman2018spectral}, ``for matrix-valued weights, their
spectral norm was computed by evaluating the eigenvalue of the eigenvector
obtained with 10 power iterations. For vector-valued weights, we simply
use the $L_{2}$ norm''.} increases gradually during training with a small-size DNN empirically.
Therefore, the inequality implies that ``longer training allows the
network to represent more complex functions by allowing it to also
fit higher frequencies''. However, for a large-size DNN, the spectral
norm almost does not change during the training (See an example in
Fig.\ref{fig:SpectralNorm} in Appendix.), that is, the bound of the
amplitude of each frequency component of the DNN output almost does
not change during the training. Then, the inequality in \cite{rahaman2018spectral}
cannot explain why the F-Principle still holds for a large-size DNN.

\paragraph*{Contribution}

In this work, we develop a theoretical framework by Fourier analysis
aiming to understand the training process and the generalization ability
of DNNs with sufficient neurons and hidden layers. We show that for
any parameter, the gradient descent magnitude in each frequency component
of the loss function is proportional to the product of two factors:
one is a decay term with respect to (w.r.t.) frequency; the other
is the amplitude of the difference between the DNN output and the
target function. This theoretical framework shows that DNNs trained
by gradient-based methods endow low-frequency components with higher
priority during the training process. Since the power spectrum of
the tanh function exponentially decays w.r.t. frequency, in which
the exponential decay rate is proportional to the inverse of weight.
We then show that small (large) initialization would result in small
(large) amplitude of high-frequency components, thus leading the DNN
output to a low (high) complexity function with good (bad) generalization
ability. Therefore, with small initialization, sufficient large DNNs
can fit any function (\cite{cybenko1989approximation}) while keeping
good generalization. 

We demonstrate that the analysis in this work can be qualitatively
extended to general DNNs. We exemplified our theoretical results through
DNNs fitting natural images, 1-d functions and MNIST dataset (\cite{lecun1998mnist}). 

The paper is established as follows. The common settings of DNNs in
this work are presented in Section \ref{sec:Methods}. The theoretical
framework is given in Section \ref{sec:Theoretical-framework}. We
then study the evolution of the mean magnitude of DNN parameters during
the DNN training empirically in Section \ref{sec:The-magnitude-of}.
The theoretical framework is validated by experiments in Section \ref{sec:Understanding-deep-learning}.
The conclusions and discussions are followed in Section \ref{sec:Discussions}.

\section{Methods\label{sec:Methods}}

The activation function for each neuron is tanh. We use DNNs of multiple
hidden layers with no activation function for the output layer. The
DNN is trained by Adam optimizer (\cite{kingma2014adam}). Parameters
of the Adam optimizer are set to the default values (\cite{kingma2014adam}).
The loss function is the mean squared error of the difference between
the DNN output and the target function in the training set. 

\section{Theoretical framework\emph{\label{sec:Theoretical-framework}}}

In this section, we will develop a theoretical framework in the Fourier
domain to understand the training process of DNN. For illustration,
we first use a DNN with one hidden layer with tanh function $\sigma(x)$
as the activation function: 
\[
\sigma(x)=\tanh(x)=\frac{e^{x}-e^{-x}}{e^{-x}+e^{x}},\quad x\in{\rm \mathbb{R}}.
\]
The Fourier transform of $\sigma(wx+b)$ with $w,b\in{\rm \mathbb{R}}$
is,
\begin{equation}
F[\sigma(wx+b)](k)=\sqrt{\frac{\pi}{2}}\delta(k)+\sqrt{\frac{\pi}{2}}\frac{i}{|w|}\exp\left(-ibk/w\right)\frac{1}{\exp(\pi k/2w)-\exp(-\pi k/2w)},\label{eq:FSigOri}
\end{equation}
 where
\[
F[\sigma(x)](k)=\int_{-\infty}^{\infty}\sigma(x)\exp(-ikx){\rm d}x.
\]
 Consider a DNN with one hidden layer with $N$ nodes, 1-d input $x$
and 1-d output:
\begin{equation}
\Upsilon(x)=\sum_{j=1}^{N}a_{j}\sigma(w_{j}x+b_{j}),\quad a_{j},w_{j},b_{j}\in{\rm \mathbb{R}}.\label{eq: DNNmath}
\end{equation}
Note that we call all $w_{j}$, $a_{j}$ and $b_{j}$ as \emph{parameters},
in which $w_{j}$ and $a_{j}$ are \emph{weights}, and $b_{j}$ is
a \emph{bias} \emph{term}. When $|\pi k/w_{j}|$ is large, without
loss of generality, we assume $\pi k/w_{j}\gg0$ and $w_{j}>0$,
\begin{equation}
F[\Upsilon](k)\approx\sum_{j=1}^{N}a_{j}\left[\sqrt{\frac{\pi}{2}}\frac{i}{w_{j}}\exp\left(-(ib_{j}+\pi/2)k/w_{j}\right)\right].\label{eq:FTW}
\end{equation}
We define the difference between DNN output and the \emph{target function}
$f(x)$ at each $k$ as 
\[
D(k)\triangleq F[\Upsilon](k)-F[f](k).
\]
Write $D(k)$ as 
\begin{equation}
D(k)=A(k)e^{i\theta(k)},\label{eq:DAT}
\end{equation}
where $A(k)$ and $\theta(k)\in[-\pi,\pi]$, indicate the amplitude
and phase of $D(k)$, respectively. The loss at frequency $k$ is
$L(k)=\frac{1}{2}\left|D(k)\right|^{2}$, where $|\cdot|$ denotes
the norm of a complex number. The total loss function is defined as:
\[
L=\sum_{k}L(k).
\]
Note that according to the Parseval's theorem\footnote{Without loss of generality, the constant factor that is related to
the definition of corresponding Fourier Transform is ignored here.}, this loss function in the Fourier domain is equal to the commonly
used loss of mean squared error, that is, 
\begin{equation}
L=\frac{1}{2}\sum_{x}(\Upsilon(x)-f(x))^{2},\label{eq:Loss1}
\end{equation}
 At frequency $k$, the amplitude of the gradient with respect to
each parameter can be obtained (see Appendix for details),
\begin{align}
\frac{\partial L(k)}{\partial a_{j}}=i\left(C_{1}-1\right)\frac{C_{0}}{w_{j}}A(k)\exp\left(-\pi k/2w_{j}\right)\label{eq:DLalpha}
\end{align}
\begin{equation}
\frac{\partial L(k)}{\partial w_{j}}=C_{0}C_{2}\frac{a_{j}}{2w_{j}^{3}}A(k)\exp\left(-\pi k/2w_{j}\right)\label{eq:paritialLA}
\end{equation}
\begin{equation}
\frac{\partial L(k)}{\partial b_{j}}=\left(C_{1}-1\right)C_{0}\frac{a_{j}k}{w_{j}^{2}}A(k)\exp\left(-\pi k/2w_{j}\right),\label{eq:paritialLB}
\end{equation}
 where 
\begin{equation}
C_{0}=\sqrt{\frac{\pi}{2}}e^{i\left[\theta(k)+b_{j}k/w_{j}\right]}
\end{equation}
\begin{equation}
C_{1}=\exp\left(-2i\left(b_{j}k/w_{j}+\theta(k)\right)\right),\label{eq:C1}
\end{equation}
\begin{equation}
C_{2}=\left[C_{1}\left(i(\pi k-2w_{j})-2b_{j}k\right)+\left(-i(\pi k-2w_{j})-2b_{j}k\right)\right],\label{eq:C2}
\end{equation}
 The descent amount at any direction, say, with respect to parameter
$\Theta_{jl}$, is 
\begin{equation}
\frac{\partial L}{\partial\Theta_{jl}}=\sum_{l}\frac{\partial L(k)}{\partial\Theta_{jl}}.\label{eq:GDfreq}
\end{equation}
The absolute contribution from frequency $k$ to this total amount
at $\Theta_{jl}$ is 
\begin{equation}
\left|\frac{\partial L(k)}{\partial\Theta_{jl}}\right|=A(k)\exp\left(-|\pi k/2w_{j}|\right)G_{jl}(\Theta_{j},k),\label{eq:DL2}
\end{equation}
where $\Theta_{j}\triangleq\{w_{j},b_{j},a_{j}\}$, $\Theta_{jl}\in\Theta_{j}$,
$G_{jl}(\Theta_{j},k)$ is a function with respect to $\Theta_{j}$
and $k$, which can be found in one of Eqs. (\ref{eq:DLalpha}, \ref{eq:paritialLA},
\ref{eq:paritialLB}).

When the component at frequency $k$ does not converge yet, $\exp\left(-|\pi k/2w_{j}|\right)$
would dominate $G_{jl}(\Theta_{j},k)$ for a small $w_{j}$. Therefore,
the behavior of Eq. (\ref{eq:DL2}) is dominated by $A(k)\exp\left(-|\pi k/2w_{j}|\right)$.
This dominant term also indicates that weights are much more important
than bias terms, which will be verified by MNIST dataset later. 

To examine the convergence behavior of different frequency components
during the training, we compute the relative difference of the DNN
output and $f(x)$ in the frequency domain at \emph{each recording
step}, that is, 
\begin{equation}
\Delta_{F}(k)=\frac{|F[f](k)-F[\Upsilon](k)|}{|F[f](k)|}.\label{eq:dfreq}
\end{equation}
Through the above analysis framework, we have the following theorems
(The proofs can be found at Appendix.)
\begin{thm}
\label{thm:Priority}Consider a DNN with one hidden layer using tanh
function $\sigma(x)$ as the activation function. For any frequencies
$k_{1}$ and $k_{2}$ such that $k_{2}>k_{1}>0$ and there exist $c_{1},c_{2},$
such that $A(k_{1})>c_{1}>0$, $A(k_{2})<c_{2}<\infty$, we have 
\begin{equation}
\lim_{\delta\rightarrow0}\frac{\mu\left(\left\{ w_{j}:\left|\frac{\partial L(k_{1})}{\partial\Theta_{jl}}\right|>\left|\frac{\partial L(k_{2})}{\partial\Theta_{jl}}\right|\quad\text{for all}\quad j,l\right\} \cap B_{\delta}\right)}{\mu(B_{\delta})}=1,\label{eq:thm1proof-1}
\end{equation}
where $B_{\delta}$ is a ball with radius $\delta$ centered at the
origin and $\mu(\cdot)$ is the Lebesgue measure of a set. 
\end{thm}

Thm \ref{thm:Priority} shows that for any two non-converged frequencies,
almost all sufficiently small weights satisfy that a lower-frequency
component has a higher priority during the gradient-based training. 
\begin{thm}
\label{thm:equalprioritySmall}Consider a DNN with one hidden layer
with tanh function $\sigma(x)$ as the activation function. For any
frequencies $k_{1}$ and $k_{2}$ such that $k_{2}>k_{1}>0$, we consider
non-degenerate situation, that is, $\left|F[f](k_{1})\right|>0$,
and there exist positive constants $C_{a}$, $\xi$, $\xi_{1}$, and
$\xi_{2}$, such that $A(k_{2})<C_{a}$, $\left|1-C_{1}(k_{1})\right|>\xi$,
and $\xi_{2}>|C_{2}(k_{1})|>\xi_{1}$. $\forall\epsilon>0$. Then,
for any $\epsilon>0$, there exists $M>0$, for any $w_{j}\in[-M,M]\backslash\{0\}$
such that there exists $\Theta_{jl}\in\{w_{j},b_{j},a_{j}\}$ satisfying
\begin{equation}
\left|\frac{\partial L(k_{1})}{\partial\Theta_{jl}}\right|=\left|\frac{\partial L(k_{2})}{\partial\Theta_{jl}}\right|,\label{eq:neq0-1}
\end{equation}
we have $\Delta_{F}(k_{1})<\epsilon$.
\end{thm}

The case that the gradient amplitude of the low-frequency component
equals to that of the high-frequency component indicates that low
frequency does not dominate. Thm \ref{thm:equalprioritySmall} implies
that when the deviation of the high-frequency component ($A(k_{2})$)
is bounded, we can always find small weights such that the relative
error of the low-frequency component stays very small when the high-frequency
component has the similar priority as the low-frequency component.
In another word, when the DNN starts to fit the high-frequency component,
the low-frequency one stays at the converged state.
\begin{thm}
\label{thm:equalpriorityHigh} Consider a DNN with one hidden layer
with tanh function $\sigma(x)$ as the activation function. For any
frequencies $k_{1}$ and $k_{2}$ such that $k_{2}>k_{1}>0$, we consider
non-degenerate situation, that is, $\left|F[f](k_{1})\right|>0$,
and there exist positive constants $\xi$, $\xi_{1}$, and $\xi_{2}$,
such that $\left|1-C_{1}(k_{1})\right|>\xi$ and $\xi_{2}>|C_{2}(k_{1})|>\xi_{1}$,
for any $B>0$, there exists $M>0$, for any $A(k_{2})>M$, such that
there exists $\Theta_{jl}\in\{w_{j},b_{j},a_{j}\}$ satisfying
\begin{equation}
\left|\frac{\partial L(k_{1})}{\partial\Theta_{jl}}\right|=\left|\frac{\partial L(k_{2})}{\partial\Theta_{jl}}\right|\neq0,\label{eq:neq02-1}
\end{equation}
we have $\Delta_{F}(k_{1})>B$.
\end{thm}

Thm \ref{thm:equalpriorityHigh} implies that for a large $A(k_{2})$,
when the low-frequency component does not converge yet, it already
can be significantly affected by the high-frequency component. A large
$A(k_{2})$ can exist when the target function is high-frequency dominate,
that is, a higher-frequency component has a larger amplitude (See
an example in Appendix \ref{subsec:Fitting-high-frequency-dominant}).

Next, we demonstrate that the analysis of Eq. (\ref{eq:DL2}) can
be qualitatively extended to general DNNs. $A(k)$ in Eq. (\ref{eq:DL2})
comes from the square operation in the loss function, thus, is irrelevant
with DNN structure. $\exp\left(-|\pi k/2w_{j}|\right)$ in Eq. (\ref{eq:DL2})
comes from the exponential decay of the activation function in the
Fourier domain. The analysis of $\exp\left(-|\pi k/2w_{j}|\right)$
is insensitive to the following factors. i) Activation function. The
power spectrum of most activation functions decreases as the frequency
increases; ii) Neuron number. The summation in Eq. (\ref{eq: DNNmath})
does not affect the exponential decay; iii) Multiple hidden layers.
If there are multiple hidden layers, the composition of continuous
activation functions is still a continuous function. The power spectrum
of the continuous function still decays as the frequency increases;
iv) High-dimensional input. We need to consider the vector form of
$k$ in Eq. (\ref{eq:FSigOri}); v) High-dimensional output. The total
loss function is the summation of the loss of each output node. Therefore,
the analysis of $A(k)\exp\left(-|\pi k/2w_{j}|\right)$ of a single
hidden layer qualitatively applies to different activation functions,
neuron numbers, multiple hidden layers, and high-dimensional functions.

\section{The magnitude of DNN parameters during training\label{sec:The-magnitude-of}}

Since the magnitude of DNN parameters is important to the analysis
of the gradients, such as $w_{j}$ in Eq. (\ref{eq:DL2}), we study
the evolution of the magnitude of DNN parameters during training.
Through training DNNs by MNIST dataset, empirically, we show that
for a network with sufficient neurons and layers, the mean magnitude
of the absolute values of DNN parameters only changes slightly during
the training. For example, we train a DNN by MNIST dataset with different
initialization. In Fig.\ref{fig:MNISTnorm}, DNN parameters are initialized
by Gaussian distribution with mean $0$ and standard deviation 0.06,
0.2, 0.6 for (a, b, c), respectively. We compute the mean magnitude
of absolute weights and bias terms. As shown in Fig.\ref{fig:MNISTnorm},
the mean magnitude of the absolute value of DNN parameters only changes
slightly during the training. Thus, empirically, the initialization
almost determines the magnitude of DNN parameters. Note that the magnitude
of DNN parameters can have significant change during training when
the network size is small (We have more discussion in \emph{Discussion}).
\begin{center}
\begin{figure}
\begin{centering}
\includegraphics[scale=0.6]{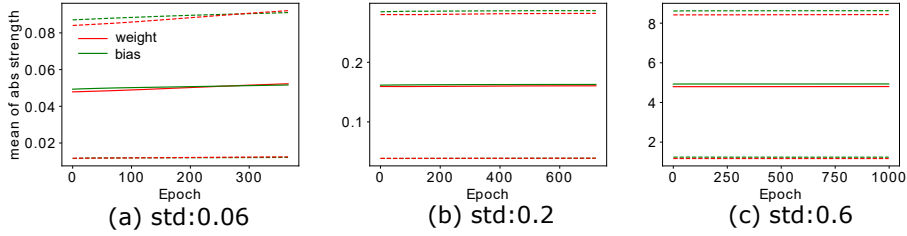}
\par\end{centering}
\caption{Magnitude of DNN parameters during fitting MNIST dataset. DNN parameters
are initialized by Gaussian distribution with mean $0$ and standard
deviation 0.06, 0.2, 0.6 for (a, b, c), respectively. Solids lines
show the mean magnitude of the absolute weights (red) and the absolute
bias (green) at each training epoch. The dashed lines are the mean$\pm$std
for the corresponding color. Note that the green and the red lines
almost overlap. We use a tanh DNN with width: 800-400-200-100. The
learning rate is $10^{-5}$ with batch size $400$. \label{fig:MNISTnorm}}
 
\end{figure}
\par\end{center}

\section{Experiments\label{sec:Understanding-deep-learning}}

Here, we consider only DNNs with sufficient neurons and layers. Since
initialization is very important to deep learning, we will first discuss
small initialization and then discuss large initialization. 

\subsection{Small initialization\label{subsec:Fitting-low-frequency-dominant} }

To show that the F-Principle holds in real data, we train a DNN to
fit a natural image, as shown in Fig.\ref{fig:LowDominate}a---a
mapping from position $(x,y)$ to gray scale strength, which is subtracted
by its mean and then normalized by the maximal absolute value. As
an illustration of F-Principle, we study the Fourier transform of
the image with respect to $x$ for a fixed $y$ (red dashed line in
Fig.\ref{fig:LowDominate}a, denoted as the \emph{target function}
$f(x)$ in the spatial domain). In a finite interval, the frequency
components of the target function can be quantified by Fourier coefficients
computed from Discrete Fourier Transform (DFT). Note that the frequency
in DFT is discrete. The Fourier coefficient of $f(x)$ for the frequency
component $k$ is denoted by $F[f](k)$ (a complex number in general).
$|F[f](k)|$ is the corresponding amplitude. Fig.\ref{fig:LowDominate}b
displays the first $40$ frequency components of $|F[f](k)|$. As
the relative error shown in Fig.\ref{fig:LowDominate}c, the first
four frequency peaks converge from low to high in order. This is consistent
with Thm \ref{thm:Priority}. In addition, the relative difference
of low-frequency components stays small when the DNN is capturing
high-frequency components. This is consistent with Thm \ref{thm:equalprioritySmall}.
Due to the small initialization, as an example in Fig.\ref{fig:LowDominate}d,
when the DNN is fitting low-frequency components, high-frequency components
stay relatively small. 

Viewing from the snapshots during the training process, we can see
that DNN captures the image from coarse-grained low-frequency components
to detailed high-frequency components (Fig.\ref{fig:LowDominate}e-\ref{fig:LowDominate}g).
With small weights, the DNN output generalizes well (as shown in Fig.\ref{fig:LowDominate}h),
that is, the test error is very close to the training error.

In addition, this theoretical framework can also explain the more
complicated training behavior of DNNs' fitting of high-frequency dominant
functions (Thm \ref{thm:equalpriorityHigh} and an example in Appendix
\ref{subsec:Fitting-high-frequency-dominant}).
\begin{center}
\begin{figure}
\begin{centering}
\includegraphics[scale=0.7]{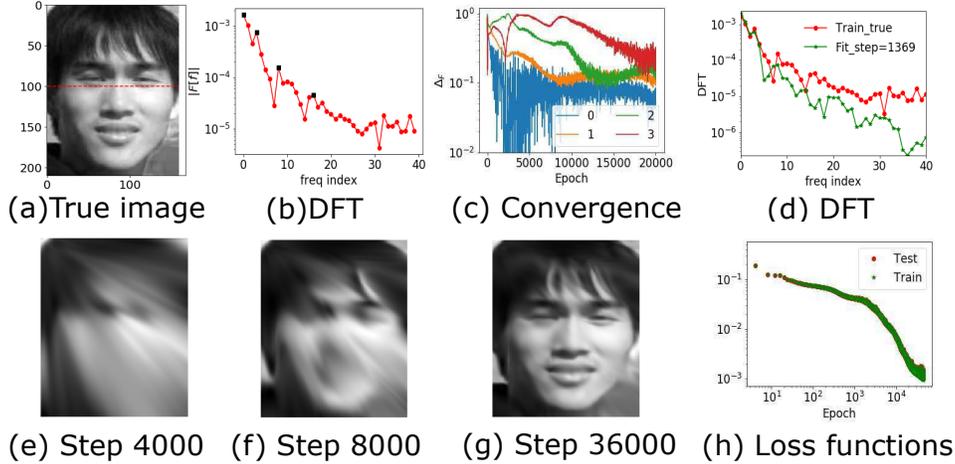}
\par\end{centering}
\caption{Convergence from low to high frequency for a natural image. The training
data are all pixels whose horizontal indexes are odd. (a) True image.
(b) $|F[f]|$ of the red dashed pixels in (a) as a function of frequency
index---Note that for DFT, we can refer to a frequency component
by the \emph{frequency index} instead of its physical frequency---with
selected peaks marked by black dots. (c) $\Delta_{F}$ at different
training epochs for different selected frequency peaks in (b). (d)
$|F[f]|$ (red) and $|F[\Upsilon]|$ (green) at epoch 1369. (e-g)
DNN outputs of all pixels at different training steps. (h) Loss functions.
We use a DNN with width 500-400-300-200-200-100-100. We train the
DNN with the full batch and learning rate $2\times10^{-5}$. We initialize
DNN parameters by Gaussian distribution with mean $0$ and standard
deviation $0.08$. \label{fig:LowDominate}}
 
\end{figure}
\par\end{center}

\subsection{Large initialization}

When the DNN is fitting the natural image in Fig.\ref{fig:LowDominate}a
with small initialization, it generalizes well as shown in Fig.\ref{fig:LowDominate}h.
Here, we show the case of large initialization. Except for a larger
standard deviation of Gaussian distribution for initialization, other
parameters are the same as in Fig.\ref{fig:LowDominate}. After training,
the DNN can well capture the training data, as shown in the left in
Fig.\ref{fig:largeInitSlow}a. However, the DNN output at the test
pixels are very noisy, as shown in the right in Fig.\ref{fig:largeInitSlow}a.
The loss functions of the training data and the test data in Fig.\ref{fig:largeInitSlow}b
also show that the DNN generalizes poorly. To visualize the high-frequency
components of DNN output after training, we study the pixels at the
red dashed lines in Fig.\ref{fig:largeInitSlow}a. As shown in the
left in Fig.\ref{fig:largeInitSlow}c, the DNN accurately fits the
training data. However, for the test data, the DNN output fluctuates
a lot, as shown in Fig.\ref{fig:largeInitSlow}d. Compared with the
case of small initialization, as shown in Fig.\ref{fig:largeInitSlow}e,
the convergence order of the first four frequency peaks are not very
clean. This is consistent with Thm \ref{thm:equalpriorityHigh}.
\begin{center}
\begin{figure}
\begin{centering}
\includegraphics[scale=0.6]{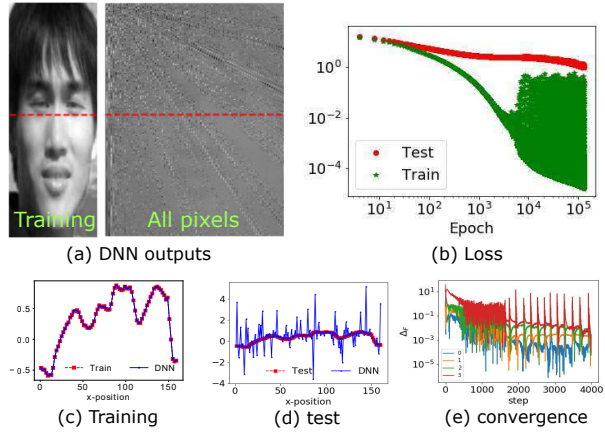}
\par\end{centering}
\caption{Analysis of the training process of DNNs with large initialization
while fitting the image in Fig.\ref{fig:LowDominate}a. The weights
of DNNs are initialized by a Gaussian distribution with mean $0$
and standard deviation $0.5$. (a) The DNN outputs at the training
pixels (left) and all pixels (right). (b) Loss functions. (c) DNN
outputs of training data at the red dashed position in (a). (d) DNN
outputs including test data at the red dashed position in (a). (e)
$\Delta_{F}$ at different training epochs for different selected
frequency peaks in Fig.\ref{fig:LowDominate}b. \label{fig:largeInitSlow}}
 
\end{figure}
\par\end{center}

\subsection{Schematic analysis}

Different initialization can result in very different generalization
ability of DNN's fitting and very different training courses. Here,
we schematically analyze DNN's output after training. With finite
training data points, there is an effective frequency range for this
training set, which is defined as the range in frequency domain bounded
by Nyquist-Shannon sampling theorem (\cite{shannon1949communication})
when the sampling is evenly spaced, or its extensions (\cite{yen1956nonuniform,mishali2009blind})
otherwise. Based on the effective frequency range, we can decompose
the Fourier transform of DNN's output into two parts, that is, effective
frequency range and \emph{extra-higher} frequency range. For different
initialization, since the DNN can well fit the training data, the
frequency components in the effective frequency range are the same. 

Then, we consider the frequency components in the extra-higher frequency
range. The amplitude at each frequency for node $j$ is controlled
by $\exp\left(-|\pi k/2w_{j}|\right)$ with a decay rate $|\pi/2w_{j}|$.
This decay rate is large for small initialization. Since the gradient
descent does not drive these extra-higher frequency components towards
any specific direction, the amplitudes of the DNN output in the extra-higher
frequency range stay small. For large initialization, the decay rate
$|\pi/2w_{j}|$ is relatively small. Then, extra-higher frequency
components of the DNN output could have large amplitudes and much
fluctuate compared with small initialization. 

Higher-frequency function is of more complexity (for example, \cite{wu2017towards}
uses $\mathbb{E}\left\Vert \nabla_{x}f\right\Vert _{2}^{2}$ to characterize
complexity). With small (large) initialization, the DNN's output is
a low-complexity (high-complexity) function after training. When the
training data captures all important frequency components of the target
function, a low-complexity DNN output can generalize much better than
a high-complexity DNN output. 

We next use MNIST dataset to verify the effect of initialization.
We use Gaussian distribution with mean $0$ to initialize DNN parameters.
For simplicity, we use a \emph{two-dimensional vector} $(\cdot,\cdot)$
to denote standard deviations of weights and bias terms, respectively.
Fix the standard deviation for bias terms, we consider the effect
of different standard deviations of weights, that is, $(0.01,0.01)$
in Fig.\ref{fig:MNIST-largeInitSlow}a and $(0.3,0.01)$ in Fig.\ref{fig:MNIST-largeInitSlow}b.
As shown in Fig.\ref{fig:MNIST-largeInitSlow}, in both cases, DNNs
have high accuracy for the training data. However, compared the red
dashed line in Fig.\ref{fig:MNIST-largeInitSlow}a with the yellow
dashed line in Fig.\ref{fig:MNIST-largeInitSlow}b, for the small
standard deviation of weight, the prediction accuracy of the test
data is much higher than that of the large one.

Note that the effect of initialization is governed by weights rather
than bias terms. To verify this, we initialize bias terms with standard
deviation $2.5$. As shown by black curves in Fig.\ref{fig:MNIST-largeInitSlow}a,
the DNN with standard deviation $(0.01,2.5)$ has a bit slower training
course and a bit lower prediction accuracy.
\begin{center}
\begin{figure}
\begin{centering}
\includegraphics{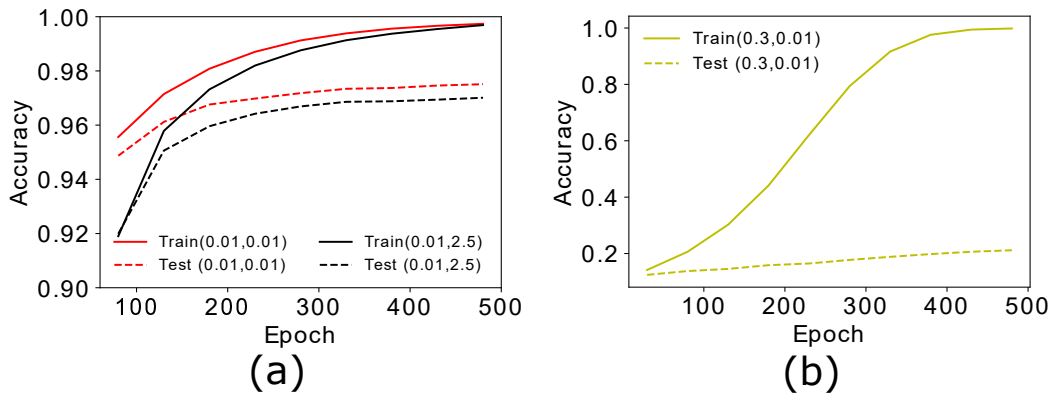}
\par\end{centering}
\caption{Analysis of the training process of DNNs with different initialization
while fitting MNIST dataset. Illustrations are the prediction accuracy
on the training data and the test data at different training epochs.
We use a tanh DNN with width: 800-400-200-100. The learning rate is
$10^{-5}$ with batch size $400$. DNN parameters are initialized
by Gaussian distribution with mean $0$. The legend $(\cdot,\cdot)$
denotes standard deviations of weights and bias terms, respectively.
\label{fig:MNIST-largeInitSlow}}
 
\end{figure}
\par\end{center}

\section{Discussions\label{sec:Discussions}}

In this work, we have theoretically analyzed the DNN training process
with (stochastic) gradient-based methods through Fourier analysis.
Our theoretical framework explains the training process when the DNN
is fitting low-frequency dominant functions or high-frequency dominant
functions (See Appendix \ref{subsec:Fitting-high-frequency-dominant}).
Based on the understanding of the training process, we explain why
DNN with small initialization can achieve a good generalization ability.
Our theoretical result shows that the initialization of weights rather
than bias terms mainly determines the DNN training and generalization.
We exemplify our results through natural images and MNIST dataset.
These analyses are not constrained to low-dimensional functions. Next,
we will discuss the relation of our results with other studies and
some limitations of these analyses. 

\paragraph{Weight norm}

In this work, we analyze the DNN with sufficient neurons and layers
such that the mean magnitude of DNN parameters keeps almost constant
throughout the training. However, if we use a small-scale network,
the mean magnitude of DNN parameters often increases to a stable value.
Empirically, we found that the training increases the magnitude of
DNN parameters when the DNN is fitting high-frequency components,
for which our theory can provide some insight. If the weights are
too small, the decay rate of $\exp\left(-|\pi k/2w_{j}|\right)$ in
Eq. (\ref{eq:FTW}) will be too large. The training will have to increase
the weights fit the high-frequency components of the target function
because the number of neurons is fixed. By imposing regularization
on the norm of weights in a small-scale network, we can prevent the
training from fitting high-frequency components. Since high-frequency
components usually have small power and are easily affected by noise,
without fitting high-frequency components, the norm regularization
will improve the DNN generalization ability. This discussion is consistent
with other studies (\cite{poggio2018theory,nagarajan2017generalization}).
We will address this topic about the evolution of the mean magnitude
of DNN parameters of small-scale networks in our the future work. 

\paragraph{Loss function and activation function}

In this work, we use the mean square error as loss function and tanh
as activation function for the training. Empirically, we found that
by using the mean absolute error as loss function, we can still observe
the convergence order from low to high frequency when the DNN is fitting
low-frequency dominant functions. The term $A(k)$ can be replaced
by other forms that can characterize the difference between DNN outputs
and the target function, by which the analysis of $A(k)$ can be extended.
The key exponential term $\exp\left(-|\pi k/2w_{j}|\right)$ comes
from the property of the activation function. Note that when computing
the Fourier transform of the activation function, we have ignored
the windowing's effect -- which would not change the property of
the activation function in the Fourier domain whose power decays as
the frequency increases. Therefore, for any activation function where
power decreases as the frequency increases and any loss function which
characterizes the difference of DNN outputs and the target function,
the analysis in this work can be qualitatively extended. The exact
mathematical forms of different activation functions and loss functions
can be different. We leave this analysis to our future work.

\paragraph{Sharp/flat minima and generalization}

Since $\exp\left(-|\pi k/2w_{j}|\right)$ exists in all gradient forms
in Eqs. (\ref{eq:DLalpha}, \ref{eq:paritialLA}, \ref{eq:paritialLB}),
$\exp\left(-|\pi k/2w_{j}|\right)$ will also exist in the second-order
derivative of the loss function with respect to any parameter. Here,
we only consider DNN parameters with similar magnitude. With smaller
weights at a minima, the DNN has a good generalization ability along
with that the second-order derivative at the minima is smaller, that
is, a flatter minima. When the weights are very large, the minima
is very sharp. When the DNN is close to a very sharp minima, one training
step can cause the loss function to deviate from the minima significantly
(See a conceptual sketch in Figure 1 of \cite{keskar2016large}).
We also observe that in Fig.\ref{fig:largeInitSlow}, for large initialization,
the loss fluctuates significantly when it is small. Our theoretical
analysis qualitatively shows that a flatter minima is associated with
a better DNN generalization, which resembles the results of other
studies (\cite{keskar2016large,wu2017towards}) (see \emph{Introduction}).

\paragraph{Early stopping }

Our theoretical framework through Fourier analysis can well explain
F-Principle, in which DNN gives low-frequency components with higher
priority as observed in \cite{xu_training_2018}. Thus, our theoretical
framework can provide insight into early stopping. High-frequency
components often have low power (\cite{dong1995statistics}) and are
noisy, but with early stopping, we can prevent DNN from fitting high-frequency
components to achieve a better generalization. 

\paragraph{Noise and real data}

Empirical studies found the qualitative differences in gradient-based
optimization of DNNs on pure noise vs. real data (\cite{zhang2016understanding,arpit2017closer,xu_training_2018}).
We would discuss the mechanism underlying these qualitative differences.

\cite{zhang2016understanding} found that the convergence time of
a DNN increases as the label corruption ratio increases. \cite{arpit2017closer}
concluded that DNNs do not use brute-force memorization to fit real
datasets but extract patterns in the data based on experimental findings
in the dataset of MNIST and CIFAR-10. \cite{arpit2017closer} suggests
that DNNs may learn simple and general patterns first in the real
data. \cite{xu_training_2018} found similar results as the following.
With the simple visualization of low-dimensional functions on Fourier
domain, \cite{xu_training_2018} found F-Principle for low-frequency
dominant functions empirically. However, F-Principle does not apply
to pure noise data (\cite{xu_training_2018}). 

To theoretically understand the above empirically findings, we first
note that the real data in Fourier domain is usually low-frequency
dominant (\cite{dong1995statistics}) while pure noise data often
do not have clear dominant frequency components, for example, white
noise. For low-frequency dominant functions, DNNs can quickly capture
the low-frequency components. However, for pure noise data, since
the high-frequency components are also important, that is, $A(k)$
in Eq. (\ref{eq:DL2}) could be large for a large $k$, the priority
of low-frequency during the training can be relatively small. During
the training, the gradients of low frequency and high frequency can
affect each other significantly. Thus, the training processes for
real data and pure noise data are often very different. Therefore,
along with the analysis in \emph{Results}, F-Principle thus can well
apply to low-frequency dominant functions but not pure noise data
(\cite{arpit2017closer,xu_training_2018}). Since the DNN needs to
capture more large-amplitude higher-frequency components when it is
fitting pure noise data, it often requires longer convergence time
(\cite{zhang2016understanding}). 

\paragraph{Memorization vs. generalization}

Traditional learning theory---which restricts capacity (such as VC
dimension (\cite{vapnik1999overview})) to achieve good generalization---cannot
explain how DNNs can have large capacity to memorize randomly labeled
dataset (\cite{zhang2016understanding}), but still possess good generalization
in real dataset. Another study has shown that sufficient large-scale
DNNs can potentially approximate any function (\cite{cybenko1989approximation}).
In this work, our theoretical framework further resolves this puzzle
by showing that both the DNN structure and the training dataset affect
the effective capacity of DNNs. In the Fourier domain, high-frequency
components can increase the DNN capacity and complexity. With small
initialization, DNNs invoke frequency components from low to high
to fit training data and keep other extra-higher frequency components,
which are beyond the effective frequency range of training data, small.
Thus, DNN's learned capacity and complexity are determined by the
training data, however, they still have the potential of large capacity.
This effect raised from individual dataset is consistent with the
speculation from \cite{kawaguchi2017generalization} that suggests
the generalization ability of deep learning is affected by different
datasets.

\paragraph{Limitations}

i) The theoretical framework in its current form could not analyze
how different DNN structures, such as convolutional neural networks,
affect the detailed training and generalization ability of DNNs. We
believe that to consider the effect of DNN structure, we need to consider
more properties of dataset in addition to that its power decays as
the frequency increases. ii) This qualitative framework cannot analyze
the difference between the depth of layers and the width of layers.
To this end, we need an exact mathematical form for the DFT of the
output of the DNN with multiple hidden layers, which will be left
for our future work. iii) This theoretical framework in its current
form cannot analyze the DNN behavior around local minima/saddle points
during training.

\section*{Acknowledgments}

The author wants to thank Tao Luo for helpful discussion on formalizing
Theorem 1, Wei Dai, Qiu Yang, Shixiao Jiang for critically proofreading
the manuscript. This work was funded by the NYU Abu Dhabi Institute
G1301.

\appendix
\begin{center}
\begin{figure}
\begin{centering}
\includegraphics[scale=0.85]{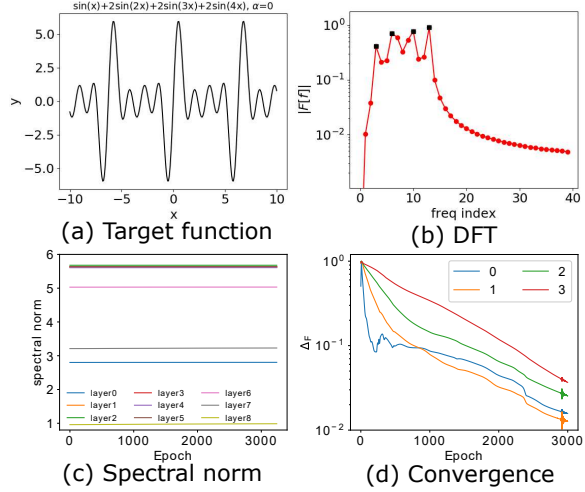}
\par\end{centering}
\caption{Convergence from low frequency to high frequency for a 1-d function
while the spectral norm almost does not change. (a) The target function.
(b) $|F[f]|$ (red solid line) as a function of frequency index with
important peaks marked by black dots. (c) Spectral norm of all weights.
As the same as \cite{rahaman2018spectral}, for matrix-valued weights,
their spectral norm was computed by evaluating the eigenvalue of the
eigenvector. For vector-valued weights, we simply use the $L_{2}$
norm. (d) $\Delta_{F}$ at different recording steps for different
selected frequency peaks in (b). The training data are evenly sampled
in $[-10,10]$ with sample size 120. We use a DNN with width: 800-800-800-800-800-800-500-100.
We train the DNN with the full batch and learning rate $2\times10^{-6}$.
We initialize DNN parameters by Gaussian distribution with mean $0$
and standard deviation $0.1$. \label{fig:SpectralNorm}}
 
\end{figure}
\par\end{center}

\section{Gradient calculation}

By definition, 

\[
\frac{\partial L(k)}{\partial a_{j}}=\overline{D(k)}\frac{\partial D(k)}{\partial a_{j}}+D(k)\frac{\partial\overline{D(k)}}{\partial a_{j}},
\]
where
\[
\frac{\partial D(k)}{\partial a_{j}}=\frac{\partial F[\Upsilon](k)}{\partial a_{j}}=\sqrt{\frac{\pi}{2}}\frac{i}{w_{j}}\exp\left(-(ib_{j}+\pi/2)k/w_{j}\right).
\]
Then
\begin{align*}
\frac{\partial L(k)}{\partial a_{j}} & =A(k)\sqrt{\frac{\pi}{2}}\frac{i}{w_{j}}\left[e^{-i\theta(k)}\exp\left(-(ib_{j}+\pi/2)k/w_{j}\right)-e^{i\theta(k)}\exp\left(-(-ib_{j}+\pi/2)k/w_{j}\right)\right]\\
 & =A(k)\sqrt{\frac{\pi}{2}}\frac{i}{w_{j}}e^{i\theta(k)}\exp\left(-\pi k/2w_{j}\right)\exp\left(ib_{j}k/w_{j}\right)\left[\exp\left(-2i(b_{j}k/w_{j}+\theta(k))\right)-1\right]
\end{align*}
Denote
\[
C_{0}=\sqrt{\frac{\pi}{2}}e^{i\left[\theta(k)+b_{j}k/w_{j}\right]},
\]
\[
C_{1}=\exp\left(-2i\left(b_{j}k/w_{j}+\theta(k)\right)\right).
\]
We have
\[
\frac{\partial L(k)}{\partial a_{j}}=i\left(C_{1}-1\right)\frac{C_{0}}{w_{j}}A(k)\exp\left(-\pi k/2w_{j}\right).
\]
Consider $w_{j}$,

\begin{align*}
\frac{\partial D(k)}{\partial w_{j}} & =\frac{\partial F[\Upsilon](k)}{\partial w_{j}}\\
 & =\sqrt{\frac{\pi}{2}}\frac{a_{j}}{2w_{j}^{3}}\exp\left(-(ib_{j}+\pi/2)k/w_{j}\right)\left(i(\pi k-2w_{j})-2b_{j}k\right),
\end{align*}
\[
\frac{\partial L(k)}{\partial w_{j}}=A(k)C_{0}\frac{a_{j}}{2w_{j}^{3}}\exp\left(-\pi k/2w_{j}\right)\left[C_{1}\left(i(\pi k-2w_{j})-2b_{j}k\right)+\left(-i(\pi k-2w_{j})-2b_{j}k\right)\right].
\]
 Denote
\[
C_{2}=\left[C_{1}\left(i(\pi k-2w_{j})-2b_{j}k\right)+\left(-i(\pi k-2w_{j})-2b_{j}k\right)\right],
\]
we have
\[
\frac{\partial L(k)}{\partial w_{j}}=C_{0}C_{2}\frac{a_{j}}{2w_{j}^{3}}A(k)\exp\left(-\pi k/2w_{j}\right).
\]
Consider $b_{j}$,
\begin{align*}
\frac{\partial D(k)}{\partial b_{j}} & =\frac{\partial F[\Upsilon](k)}{\partial b_{j}}\\
 & =\sqrt{\frac{\pi}{2}}\frac{a_{j}}{w_{j}^{2}}\exp\left(-(ib_{j}+\pi/2)k/w_{j}\right)k
\end{align*}
\[
\frac{\partial L(k)}{\partial b_{j}}=\left(C_{1}-1\right)C_{0}\frac{a_{j}k}{w_{j}^{2}}A(k)\exp\left(-\pi k/2w_{j}\right)
\]

\section{Proof of Theorem \ref{thm:Priority}\label{sec:Proof-of-theorem1}}
\begin{thm*}
\label{thm:Priority-1}Consider a DNN with one hidden layer using
tanh function $\sigma(x)$ as the activation function. For any frequencies
$k_{1}$ and $k_{2}$ such that $k_{2}>k_{1}>0$ and there exist $c_{1},c_{2},$
such that $A(k_{1})>c_{1}>0$, $A(k_{2})<c_{2}<\infty$, we have 
\begin{equation}
\lim_{\delta\rightarrow0}\frac{\mu\left(\left\{ w_{j}:\left|\frac{\partial L(k_{1})}{\partial\Theta_{jl}}\right|>\left|\frac{\partial L(k_{2})}{\partial\Theta_{jl}}\right|\quad\text{for all}\quad j,l\right\} \cap B_{\delta}\right)}{\mu(B_{\delta})}=1,\label{eq:thm1proof}
\end{equation}
where $B_{\delta}$ is a ball with radius $\delta$ centered at the
origin and $\mu(\cdot)$ is the Lebesgue measure of a set. 
\end{thm*}
\begin{proof}
First, we prove when $\Theta_{jl}=a_{j}$. For a $\xi>0$, we begin
by considering the case when $\left|1-C_{1}(k_{1})\right|\in[0,\xi]$.
Denote $\gamma=2\left(b_{j}k_{1}/w_{j}-\theta(k_{1})\right)$. When
$\xi$ is very close to zero, $\gamma$ is close to $2n\pi$, $n\in\mathbb{\mathbb{Z}}$.
Denote $u\in[0,\pi]$ such that 
\[
|1-e^{iu}|=\xi.
\]
Since $\xi$ is very close to zero, $u$ is also close to zero. Then,
we have 
\[
\xi=u+o(u).
\]
The range of $\gamma$ such that $\left|1-C_{1}(k_{1})\right|\in[0,\xi]$
for each $n$ is $[2n\pi-u,2n\pi+u]$. Consider
\[
2\left(b_{j}k_{1}/w_{j}-\theta(k_{1})\right)=2n\pi+u.
\]
Our theorem considers $w_{j}\rightarrow0$, and when $w_{j}$ is very
small, we can always have $2\left|b_{j}k_{1}/w_{j}-\theta(k_{1})\right|>|u|$.
Therefore, we can only consider $n\neq0$.

By denoting $\beta_{n}=2n\pi+2\theta(k_{1})$, we have 
\[
w_{j}=\frac{2b_{j}k_{1}}{\beta_{n}+u}.
\]
Therefore, the volume of $w_{j}$ such that $\left|1-C_{1}(k_{1})\right|\in[0,\xi]$
for each $n\neq0$ is 
\[
V_{1,n}=\frac{2b_{j}k_{1}}{\beta_{n}-u}-\frac{2b_{j}k_{1}}{\beta_{n}+u}=\frac{2b_{j}k_{1}}{\beta_{n}^{2}-u^{2}}2u\leq\frac{2b_{j}k_{1}}{\beta_{n}^{2}}2u.
\]
Denote $C_{3}=4b_{j}k_{1}$, we have the Lebesgue measure of $w_{j}$
such that $\left|1-C_{1}\right|\in[0,\xi]$ as 
\[
V_{1}=\sum_{n=-\infty}^{\infty}V_{1,n}\leq\sum_{n=-\infty}^{\infty}\frac{C_{3}}{\beta_{n}^{2}}u.
\]
Since $\sum_{n=-\infty}^{\infty}\frac{1}{\beta_{n}^{2}}$ is bounded,
we have 
\[
V_{1}\leq B_{0}u,
\]
where $B_{0}$ is a constant independent of $w_{j}$.

Next, we consider the Lebesgue measure of $w_{j}$ when $\left|1-C_{1}(k_{1})\right|>\xi$
and the following holds
\begin{equation}
\left|\frac{\partial L(k_{1})}{\partial a_{j}}\right|>\left|\frac{\partial L(k_{2})}{\partial a_{j}}\right|.\label{eq:ImportantIneq}
\end{equation}
 Consider 
\[
A(k_{1})\exp\left(-|\pi k_{1}/2w_{j}|\right)\left|1-C_{1}(k_{1})\right|>A(k_{2})\exp\left(-|\pi k_{2}/2w_{j}|\right)\left|1-C_{1}(k_{2})\right|,
\]
we have 
\begin{equation}
\exp\left(\pi(k_{2}-k_{1})/|2w_{j}|\right)>\frac{A(k_{2})}{A(k_{1})}\frac{\left|1-C_{1}(k_{2})\right|}{\left|1-C_{1}(k_{1})\right|}.\label{eq:expineq}
\end{equation}
Denote $\Delta k=k_{2}-k_{1}$. Since $A(k_{1})>c_{1}>0$, $A(k_{2})<c_{2}<\infty$,
we can denote $C_{4}=\max(2A(k_{2})/A(k_{1}))$. To find $w_{j}$
such that Eq. (\ref{eq:ImportantIneq}) holds, it is sufficient to
consider
\[
\exp\left(\pi\Delta k/|2w_{j}|\right)>\max\left(\frac{A(k_{2})}{A(k_{1})}\right)\frac{\left|1-C_{1}(k_{2})\right|}{\left|1-C_{1}(k_{1})\right|},
\]
it is also sufficient to consider 
\[
\exp\left(\pi\Delta k/|2w_{j}|\right)>\max\left(\frac{A(k_{2})}{A(k_{1})}\right)\frac{2}{\xi}=\frac{C_{4}}{\xi}.
\]
Then, 
\[
\pi\Delta k/|2w_{j}|>\ln\frac{C_{4}}{\xi},
\]
\begin{equation}
|w_{j}|<\frac{\pi\Delta k}{2\left(\ln C_{4}-\ln\xi\right)}.\label{eq:asigma}
\end{equation}
Denote 
\[
G\triangleq\frac{\pi\Delta k}{2\left(\ln C_{4}-\ln\xi\right)}.
\]
Consider the interval $(0,G]$. The total Lebesgue measure is $G$.
Then, the Lebesgue measure of $w_{j}$ when $\left|1-C_{1}(k_{1})\right|>\xi$
and Eq. (\ref{eq:ImportantIneq}) holds is $V_{2}>G-V_{1}$. The ratio
\[
\lim_{\xi\rightarrow0}\frac{V_{2}}{G}\geq\lim_{\xi\rightarrow0}\left(1-\frac{V_{1}}{G}\right)=1-\lim_{\xi\rightarrow0}B_{0}\frac{\ln C_{4}-\ln\xi}{\pi\Delta k}(\xi+o(\xi))=1.
\]
Therefore, when $\xi\rightarrow0$, for all most all $|w_{j}|\in(0,G]$,
Eq. (\ref{eq:ImportantIneq}) holds. The proofs for $\Theta_{jl}=w_{j},b_{j}$
are similar. Then, we have proved the theorem.
\end{proof}

\section{Proof of Theorem \ref{thm:equalprioritySmall}}
\begin{thm*}
\label{thm:equalprioritySmall-1}Consider a DNN with one hidden layer
with tanh function $\sigma(x)$ as the activation function. For any
frequencies $k_{1}$ and $k_{2}$ such that $k_{2}>k_{1}>0$, we consider
non-degenerate situation, that is, $\left|F[f](k_{1})\right|>0$,
and there exist positive constants $C_{a}$, $\xi$, $\xi_{1}$, and
$\xi_{2}$, such that $A(k_{2})<C_{a}$, $\left|1-C_{1}(k_{1})\right|>\xi$,
and $\xi_{2}>|C_{2}(k_{1})|>\xi_{1}$. $\forall\epsilon>0$. Then,
for any $\epsilon>0$, there exists $M>0$, for any $w_{j}\in[-M,M]\backslash\{0\}$
such that there exists $\Theta_{jl}\in\{w_{j},b_{j},a_{j}\}$ satisfying
\begin{equation}
\left|\frac{\partial L(k_{1})}{\partial\Theta_{jl}}\right|=\left|\frac{\partial L(k_{2})}{\partial\Theta_{jl}}\right|,\label{eq:neq0}
\end{equation}
we have $\Delta_{F}(k_{1})<\epsilon$.
\end{thm*}
\begin{proof}
First, we prove when $\Theta_{jl}=a_{j}$. From the condition in Eq.
(\ref{eq:neq0}), we have

\[
A(k_{1})\exp\left(-|\pi k_{1}/w_{j}|\right)\left|1-C_{1}(k_{1})\right|\frac{2\pi}{w_{j}}=A(k_{2})\exp\left(-|\pi k_{2}/w_{j}|\right)\left|1-C_{1}(k_{2})\right|\frac{2\pi}{w_{j}}
\]

\[
A(k_{1})=A(k_{2})\exp\left(-\pi(k_{2}-k_{1})/|w_{j}|\right)\frac{\left|1-C_{1}(k_{2})\right|}{\left|1-C_{1}(k_{1})\right|}
\]
\begin{align*}
\Delta_{F}(k_{1}) & =\frac{A(k_{2})}{\left|F[f](k_{1})\right|}\exp\left(-\pi(k_{2}-k_{1})/|w_{j}|\right)\frac{\left|1-C_{1}(k_{2})\right|}{\left|1-C_{1}(k_{1})\right|}\\
 & <\frac{A(k_{2})}{\left|F[f](k_{1})\right|}\exp\left(-\pi(k_{2}-k_{1})/|w_{j}|\right)\frac{2}{\xi}.
\end{align*}
Note that $C_{a}>A(k_{2})$. To make $\Delta_{F}(k_{1})<\epsilon$,
it is sufficient to consider 
\[
\frac{C_{a}}{\left|F[f](k_{1})\right|}\exp\left(-\pi(k_{2}-k_{1})/|w_{j}|\right)\frac{2}{\xi}<\epsilon,
\]
 we have 
\[
w_{j}<-\frac{\pi(k_{2}-k_{1})}{\ln\left(\epsilon C_{5}\right)},
\]
where 
\[
C_{5}=\frac{\xi\left|F[f](k_{1})\right|}{2C_{a}}.
\]
 Then, let 
\[
M_{1}=-\frac{\pi(k_{2}-k_{1})}{\ln\left(\epsilon C_{5}\right)},
\]
 For any $w_{j}$ such that $0<|w_{j}|<M_{1}$, $\Delta_{F}(k_{1})<\epsilon$.
Similarly, for $\Theta_{jl}=w_{j},b_{j}$, we can have $M_{2}$ and
$M_{3}$. Then, let 
\[
M=\min(M_{1},M_{2},M_{3}),
\]
we have that for any $w_{j}$such that $0<|w_{j}|<M$, $\Delta_{F}(k_{1})<\epsilon$.
\end{proof}

\section{Proof of Theorem \ref{thm:equalpriorityHigh}}
\begin{thm*}
\label{thm:equalpriorityHigh-1}Consider a DNN with one hidden layer
with tanh function $\sigma(x)$ as the activation function. For any
frequencies $k_{1}$ and $k_{2}$ such that $k_{2}>k_{1}>0$, we consider
non-degenerate situation, that is, $\left|F[f](k_{1})\right|>0$,
and there exist positive constants $\xi$, $\xi_{1}$, and $\xi_{2}$,
such that $\left|1-C_{1}(k_{1})\right|>\xi$ and $\xi_{2}>|C_{2}(k_{1})|>\xi_{1}$,
for any $B>0$, there exists $M>0$, for any $A(k_{2})>M$, such that
there exists $\Theta_{jl}\in\{w_{j},b_{j},a_{j}\}$ satisfying
\begin{equation}
\left|\frac{\partial L(k_{1})}{\partial\Theta_{jl}}\right|=\left|\frac{\partial L(k_{2})}{\partial\Theta_{jl}}\right|\neq0,\label{eq:neq02}
\end{equation}
we have $\Delta_{F}(k_{1})>B$.
\end{thm*}
\begin{proof}
First, we prove when $\Theta_{jl}=a_{j}$. Due to the condition Eq.
(\ref{eq:neq02}), we have

\[
A(k_{1})\exp\left(-|\pi k_{1}/w_{j}|\right)\left|1-C_{1}(k_{1})\right|\frac{2\pi}{w_{j}}=A(k_{2})\exp\left(-|\pi k_{2}/w_{j}|\right)\left|1-C_{1}(k_{2})\right|\frac{2\pi}{w_{j}}
\]

\[
A(k_{1})=A(k_{2})\exp\left(-\pi(k_{2}-k_{1})/|w_{j}|\right)\frac{\left|1-C_{1}(k_{2})\right|}{\left|1-C_{1}(k_{1})\right|}
\]
\begin{align*}
\Delta_{F}(k_{1}) & =\frac{A(k_{2})}{\left|F[f](k_{1})\right|}\exp\left(-\pi(k_{2}-k_{1})/|w_{j}|\right)\frac{\left|1-C_{1}(k_{2})\right|}{\left|1-C_{1}(k_{1})\right|}\\
 & >\frac{A(k_{2})}{\left|F[f](k_{1})\right|}\exp\left(-\pi(k_{2}-k_{1})/|w_{j}|\right)\frac{\xi}{2}.
\end{align*}
To have $\Delta_{F}(k_{1})>B$, it is sufficient to consider 
\[
\frac{A(k_{2})}{\left|F[f](k_{1})\right|}\exp\left(-\pi(k_{2}-k_{1})/|w_{j}|\right)\frac{\xi}{2}>B,
\]
we have 
\[
A(k_{2})>B\exp\left(\pi(k_{2}-k_{1})/|w_{j}|\right)\frac{\left|F[f](k_{1})\right|\xi}{2}.
\]
Then, let 
\[
M_{1}=B\exp\left(\pi(k_{2}-k_{1})/|w_{j}|\right)\frac{\left|F[f](k_{1})\right|\xi}{2},
\]
 For any $A(k_{2})>M_{1}$, $\Delta_{F}(k_{1})>B$. Similarly, for
$\Theta_{jl}=w_{j},b_{j}$, we can have $M_{2}$ and $M_{3}$. Then,
let 
\[
M=\max(M_{1},M_{2},M_{3}),
\]
we have that for any $A(k_{2})>M$, $\Delta_{F}(k_{1})>B$.
\end{proof}

\section{Fitting high-frequency dominant functions \label{subsec:Fitting-high-frequency-dominant}}

We construct a function $g(x)$ as follows. First, we  evenly sample
$40$ points from $[-1,1]$ for $f(x)=x$. Second, we perform DFT
for $f(x)$, that is, $F[f](k)$. Third, we flip $F[f](k)$ to derive
$g(x)$ as shown in Fig.\ref{fig:Flip}a and Fig.\ref{fig:Flip}b.
In Fig. \ref{fig:Flip}c, except for the highest peak (the $19$th
frequency), $\Delta_{F}(k)$ of other frequency components oscillate
and decrease their amplitudes. For a very low-frequency components,
such as the first frequency component in Fig. \ref{fig:Flip}c, its
$\Delta_{F}(k)$ oscillates more, and its amplitude is small compared
with other frequencies in Fig. \ref{fig:Flip}c. Note that in this
experiment, although low-frequency component could deviate from the
converged state, they still converge faster after deviation compared
with high-frequency ones.
\begin{center}
\begin{figure}
\begin{centering}
\includegraphics{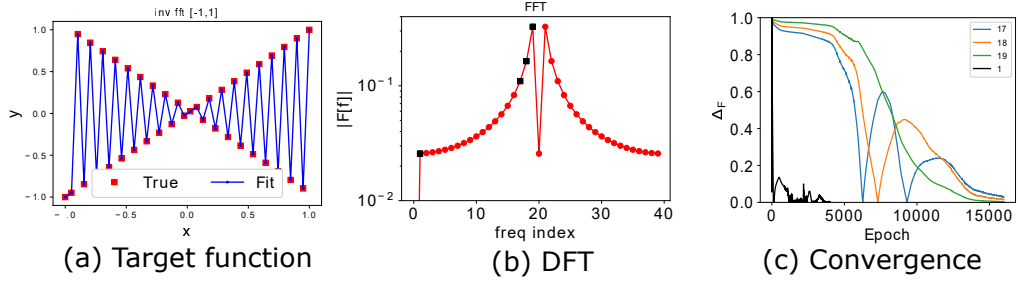}
\par\end{centering}
\caption{Frequency domain analysis for high-frequency dominant function $g(x)$,
whose DFT is obtained by flipping the DFT of $y=x$ with $40$ evenly
sampled points from $[-1,1]$. (a) red dots is for $g(x)$, blue squares
(connected by blue lines) are for the DNN output at the end of training.
(b) FFT of $g(x)$. (c) $\Delta_{F}$ at different recording steps
for different frequency peaks (different curves). We use a tanh DNN
with width: 200-200-200-200-100. The learning rate is $2\times10^{-5}$
with the full batch. DNN parameters are initialized by Gaussian distribution
with mean $0$ and standard deviation 0.1. \label{fig:Flip}}
 
\end{figure}
\par\end{center}

\newpage{}

\subsection*{\bibliographystyle{agsm}
\bibliography{C:/GDrive/Research/DeepLearning-xzx/BIB/DLRef}
}
\end{document}